\documentclass[a4paper,10pt]{article}
\usepackage[utf8x]{inputenc}

\usepackage[numbers]{natbib}

\usepackage{Definitions}

\acrodef{VMF}{von Mises Fisher}
\acrodef{KL}{Kullback Leibler}
\acrodef{PDF}{probability density function}
\acrodef{PAC}{probably approximately correct}

\title{A Note on the Kullback-Leibler Divergence for the von Mises-Fisher distribution}
\author{T.R. Diethe\\ \small Department of Electrical and Electronic Engineering\\ \small University of Bristol}

\begin{document}

\maketitle

\begin{abstract}
We present a derivation of the \acf{KL}-Divergence (also known as Relative Entropy) for the \acf{VMF} Distribution in $d-$dimensions.
\end{abstract}

\section{Introduction}
The \acf{VMF} Distribution (also known as the Langevin Distribution \cite{Watamori96}) is a probability distribution on the $(d-1)$-dimensional hypersphere $S^{d-1}$ in $\Reals^d$ \cite{Fisher53}. If $d=2$ the distribution reduces to the von Mises distribution on the circle, and if $d=3$ it reduces to the Fisher distribution on a sphere. It was introduced by \cite{Fisher53} and has been studied extensively by \cite{Mardia14,Mardia75}. The first Bayesian analysis was in \cite{Mardia76} and recently it has been used for clustering on a hypersphere by \cite{Banerjee05}.

\begin{figure}[ht]
  \centering
  \includegraphics[width=0.4\textwidth]{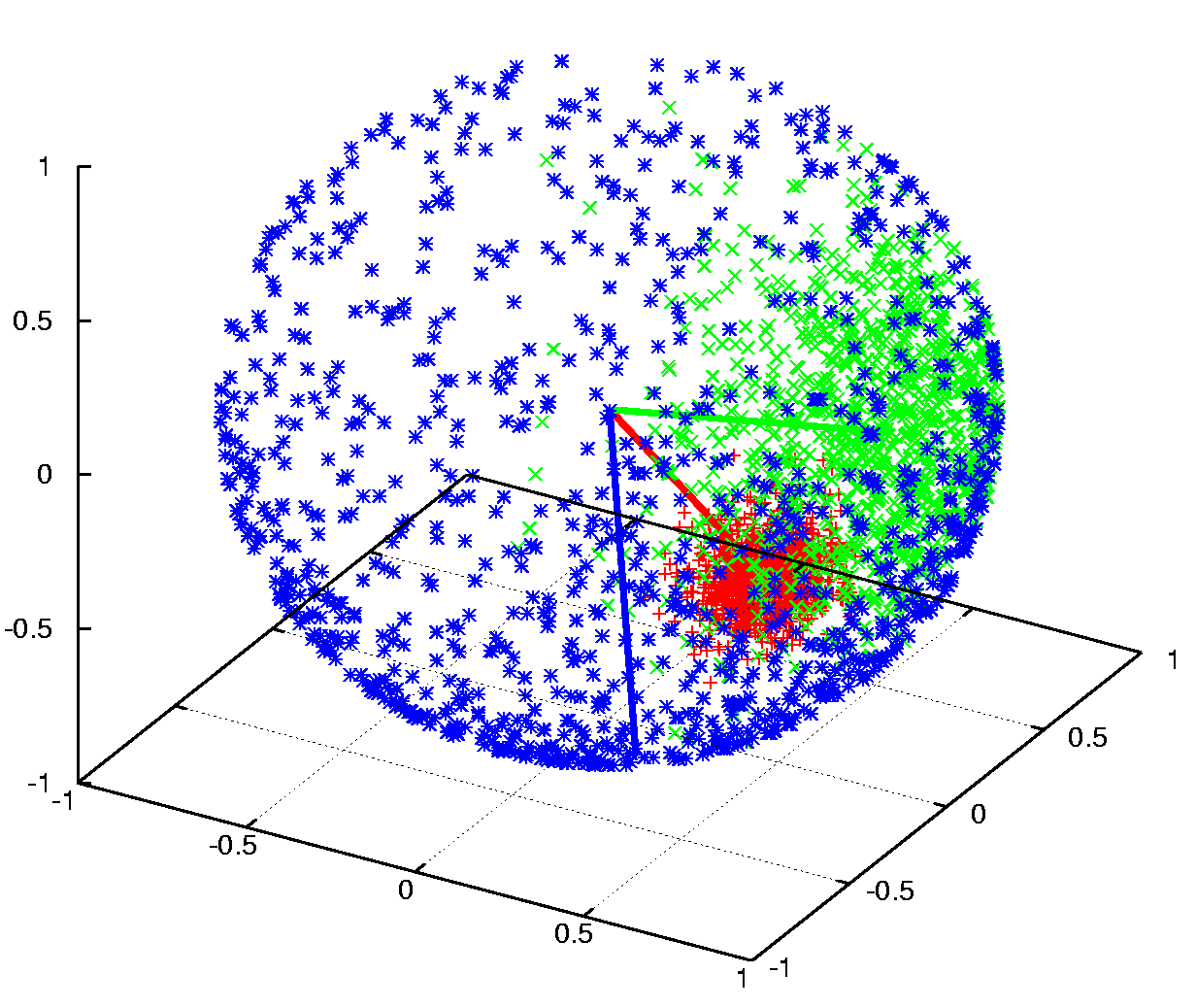}
  \caption{Three sets of 1000 points sampled from three \ac{VMF} distributions on the 3D sphere with $\kappa=1$ (blue), $\kappa=10$ (green) and $\kappa=100$ (red), respectively. The mean directions are indicated with arrows.}
\end{figure}

\section{Preliminaries}
\subsection{Definitions}
We will use $\log(z)$ to denote the natural logarithm of $z$ throughout this article. Before continuing it will be useful to define the Gamma function $\Gamma(z)$,
\begin{align}
    \Gamma(z) &= \int_0^\infty{t^{z-1}e^{-t}dt}, \quad & z\in \Complex, Re(z)>0 \\
    \Gamma(z) &= (z-1)!,                         \quad & z\in \Integers^+
\end{align}
and its relation, the incomplete Gamma function $\Gamma(z,s)$,
\begin{align}\label{eq:incomplete}
    \Gamma(z,s) & = (s-1)! e^{-x} \sum_{m=0}^{s-1}{\frac{z^m}{m!}}, \quad z\in \Integers^+
\end{align}

and the Modified Bessel Function of the First Kind $I_{\alpha}(z)$,
\begin{align}
    I_{\alpha}(z) = \sum_{m=0}^{\infty} \frac{(z/2)^{2m+\alpha}}{m! \Gamma(m+\alpha+1)},
\end{align}
which also has the following integral representations \cite{Abramowitz72},
\begin{align}
    I_{\alpha}(z) &= \frac{(z/2)^\alpha}{\sqrt{\pi}\Gamma(\alpha+1/2)} \int_0^{\pi}{e^{\pm z \cos \theta} \sin^{2d} \theta ~d\theta}, \quad &\LP \alpha \in \Reals\RP \label{eq:int1}\\
                  &= \frac{(z/2)^\alpha}{\sqrt{\pi}\Gamma(\alpha+1/2)} \int_{-1}^{1}{(1-t^2)^{(\alpha-1/2)} e^{\pm zt} ~dt}. \quad & \LP \alpha \in \Reals, \alpha > -0.5\RP \label{eq:int2}
\end{align}
Also of interest is the logarithm of this quantity (using the second integral definition \eqref{eq:int2}),
\begin{align}
 \log\LP I_{\alpha}(z) \RP &= \log \LB \frac{(\frac{z}{2})^\alpha}{\sqrt{\pi}\Gamma(\alpha+1/2)} \int_{-1}^{1}{(1-t^2)^{(\alpha-1/2)} e^{\pm zt} ~dt} \RB \notag \\
                           &= \log \frac{(\frac{z}{2})^\alpha}{\sqrt{\pi}\Gamma(\alpha+1/2)} + \log \LB \int_{-1}^{1}{(1-t^2)^{(\alpha-1/2)} e^{\pm zt} ~dt} \RB \notag \\
                           &= \log \LP\frac{z}{2}\RP^\alpha - \log \sqrt{\pi}\Gamma(\alpha+1/2) + \log \LB \int_{-1}^{1}{(1-t^2)^{(\alpha-1/2)} e^{\pm zt} ~dt} \RB .\label{eq:logMB}
\end{align}
Note that the second term does not depend on $z$.

The Exponential Integral function $E_n(z)$ is given by,
\begin{align}\label{eq:eint}
 E_{\alpha}(z) &= \int_1^{\infty}{\frac{e^{-zt}}{t^\alpha} dt}, \notag \\
        &= z^{\alpha-1}\Gamma(1-\alpha, z).
\end{align}

An identity that will be useful is,
\begin{align}\label{eq:identity1}
 \int_{-1}^1 (1-t)^d e^{t \kappa} = -2^{d-1} E_{-d}(2\kappa) e^{\kappa}. \quad d > 0
\end{align}

\subsection{The \acf{VMF} distribution}
The \ac{PDF} of the \ac{VMF} distribution for a random d-dimensional unit vector $\x (\LN \x \RN_2 = 1)$ is given by:
\begin{align}
    M_d(\mub,\kappa) = c_d(\kappa)e^{\kappa\mub'\x}, \quad \x \in S^{d-1},
\end{align}
where the normalisation constant $c_d(\kappa)$ is given by,
\begin{align}
    c_d(\kappa) = \frac{\kappa^{d/2-1}}{(2\pi)^{d/2}I_{d/2-1}(\kappa)}.
\end{align}
The (non-symmetric) \acf{KL}-Divergence from one probability distributions $q(\x)$ to another probability distribution $p(\x)$ is defined as,
\begin{align}\label{eq:kl}
    \KL(q(\x) || p(\x)) &= \int_{\x} q(\x) \log \frac{q(\x)}{p(\x)} ~d\x, \\
                        &= \EE_x \LB \log \frac{q(\x)}{p(\x)} \RB.
\end{align}
Although this is general to any two distributions, we will assume that $p(\x)$ is the ``prior'' distribution and $q(\x)$ is the ``posterior'' distribution as commonly used in Bayesian analysis.

\section{\ac{KL}-Divergence for the \ac{VMF} Distribution}
\subsection{General Case}
We will assume that we have prior and posterior distributions defined over vectors $\x \in \Reals^d, \LN \x \RN_2 = 1$ as follows,
\begin{align}
    p(\x) \sim M_d(\mub_p, \kappa_p), \notag \\
    q(\x) \sim M_d(\mub_q, \kappa_q).
\end{align}
We will now derive the \ac{KL}-Divergence for two \ac{VMF} distributions. The main problem in doing so will be the the normalisation constants $c_d(\kappa_p)$ and $c_d(\kappa_q)$.
\begin{theorem}
 For prior and posterior distributions as defined above over vectors $\x \in \Reals^d, \LN \x \RN_2 = 1, d < \infty, d$ odd\footnote{For even $d$ we can simply add a ``null'' dimension}, we have
\begin{align}
 \KL(q(\x) || p(\x)) &\leq \kappa_q - \kappa_p \mub_p'\mub_q + \dbullet\log(\kappa_q) + \sum_{m=1}^{\ddiamond}\frac{\kappa_q^m}{m!}  \notag \\
                     & \hspace{1.1cm} -\LP \frac{d^2 - 2d + 1}{4} \RP \log (\kappa_p )+ \ddiamond(\ddiamond + 1)\log \ddiamond - \ddiamond^2 + 1 \notag \\
\end{align}

\end{theorem}

\begin{proof}
From \eqref{eq:kl}, letting $\dstar = \frac{d}{2}-1$, $\ddiamond = \frac{d-3}{2}$, and $\dbullet = \frac{d-1}{2}$, we have,
{\allowdisplaybreaks 
\begin{align}
    &\KL(q(\x) || p(\x)) = \int_{\x} q(\x) \log \frac{q(\x)}{p(\x)} d\x, \notag \\
                        &= \int_{\x} q(\x) \LB \log c_d(\kappa_q)e^{\kappa_q\mub_q'\x} - \log c_d(\kappa_p)e^{\kappa_p\mub_p'\x} \RB d\x, \notag \\
                        &= \int_{\x} q(\x) \LB \log c_d(\kappa_q) - \log c_d(\kappa_p) + \kappa_q\mub_q'\x - \kappa_p\mub_p'\x \RB d\x, \notag \\
                        &= \int_{\x} q(\x) \LB \dstar\log(\kappa_q) - (d/2)\log(2\pi) - \log I_{\dstar}(\kappa_q) \right. \notag \\ 
                        & \hspace{1.3cm}\left. - \dstar\log(\kappa_p) + (d/2)\log(2\pi) + \log I_{\dstar}(\kappa_p) + \kappa_q\mub_q'\x - \kappa_p\mub_p'\x \RB d\x, \notag \\
                        &= \int_{\x} q(\x) \LB \dstar\log\LP\frac{\kappa_q}{\kappa_p}\RP - \log I_{\dstar}(\kappa_q) + \log I_{\dstar}(\kappa_p) + \kappa_q\mub_q'\x - \kappa_p\mub_p'\x \RB d\x \notag \\
                        &= \int_{\x} q(\x) \LB \dstar\log\LP\frac{\kappa_q}{\kappa_p}\RP + \kappa_q\mub_q'\x - \kappa_p\mub_p'\x \right. \notag \\
                        & \hspace{1.1cm} - \log \LP\frac{\kappa_q}{2}\RP^\dstar + \log \sqrt{\pi}\Gamma\LP\dstar+\frac{1}{2}\RP - \log \int_{-1}^{1}{(1-t^2)^{(\dstar-1/2)} e^{\pm \kappa_q t} ~dt} \notag \\
                        & \hspace{1.1cm}\left. + \log \LP\frac{\kappa_p}{2}\RP^\dstar - \log \sqrt{\pi}\Gamma\LP\dstar+\frac{1}{2}\RP + \log \int_{-1}^{1}{(1-t^2)^{(\dstar-1/2)} e^{\pm \kappa_p t} ~dt} \RB d\x \notag \\
&\mbox{(Using \eqref{eq:logMB})} \\
                        &= \int_{\x} q(\x) \LB \dstar\log\LP\frac{\kappa_q}{\kappa_p}\RP + \kappa_q\mub_q'\x - \kappa_p\mub_p'\x - \dstar \log \LP\frac{\kappa_q}{2}\RP + \dstar \log \LP\frac{\kappa_p}{2}\RP \right.\notag \\
                        & \hspace{1.1cm} \left. - \log \int_{-1}^{1}{(1-t^2)^{\ddiamond} e^{\pm \kappa_q t} ~dt} + \log \int_{-1}^{1}{(1-t^2)^{\ddiamond} e^{\pm \kappa_p t} ~dt} \RB d\x \notag \\
                        &= \int_{\x} q(\x) \LB \kappa_q\mub_q'\x - \kappa_p\mub_p'\x \right. \notag \\
                        & \hspace{1.1cm} \left. - \log \int_{-1}^{1}{(1-t^2)^{\ddiamond} e^{\pm \kappa_q t} ~dt} + \log \int_{-1}^{1}{(1-t^2)^{\ddiamond} e^{\pm \kappa_p t} ~dt} \RB d\x \notag \\
                        &= \int_{\x} q(\x) \LB \kappa_q\mub_q'\x - \kappa_p\mub_p'\x \right. \notag \\
                        & \hspace{1.1cm}\left. -\log \LB -2^{\frac{d-3}{2}} E_{-\ddiamond}(2\kappa_q) e^{\kappa_q} \RB +\log \LB -2^{\frac{d-3}{2}} E_{-\ddiamond}(2\kappa_p) e^{\kappa_p} \RB \RB d\x \notag \\
&\mbox{(Using \eqref{eq:identity1})} \\
                        &= \int_{\x} q(\x) \LB \kappa_q\mub_q'\x - \kappa_p\mub_p'\x -\kappa_q + \kappa_p - \log \LB E_{-\ddiamond}(2\kappa_q) \RB + \log \LB E_{-\ddiamond}(2\kappa_p) \RB \RB d\x \notag \\
                        &= \int_{\x} q(\x) \LB \kappa_q(\mub_q'\x - 1) - \kappa_p(\mub_p'\x - 1) \right. \notag \\
                        & \hspace{1.1cm}\left. - \log \LP 2\kappa_q^{-\dbullet} \Gamma\LP \dbullet, 2\kappa_q\RP \RP + \log \LP 2\kappa_p^{-\dbullet} \Gamma\LP \dbullet, 2\kappa_p\RP \RP \RB d\x \notag \\
&\mbox{(Using the definition of the Exponential Integral function \eqref{eq:eint})} \\
                        &= \int_{\x} q(\x) \LB \kappa_q(\mub_q'\x - 1) - \kappa_p(\mub_p'\x - 1) + \dbullet\log(2\kappa_q) - \dbullet\log(2\kappa_q)\right. \notag \\
                        & \hspace{1.1cm}\left. - \log \LP \Gamma\LP\dbullet, 2\kappa_q\RP \RP + \log \LP \Gamma\LP\dbullet, 2\kappa_p\RP \RP \RB d\x \notag \\
                        &= \int_{\x} q(\x) \LB \kappa_q(\mub_q'\x - 1) - \kappa_p(\mub_p'\x - 1) + \dbullet\log(2\kappa_q) - \dbullet\log(2\kappa_q)\right. \notag \\
                        & \hspace{1.1cm}\left. - \log \LP \ddiamond ! e^{-\kappa_q} \sum_{m=0}^{\ddiamond}\frac{\kappa_q^m}{m!} \RP
                                               + \log \LP \ddiamond ! e^{-\kappa_p} \sum_{m=0}^{\ddiamond}\frac{\kappa_p^m}{m!} \RP \RB d\x \notag \\
&\mbox{(Using \eqref{eq:incomplete} and that $\dbullet - 1 = \ddiamond$)} \\
                        &= \int_{\x} q(\x) \LB \kappa_q(\mub_q'\x - 1) - \kappa_p(\mub_p'\x - 1) + \dbullet\log(\kappa_q) - \dbullet\log(\kappa_p) + \kappa_q - \kappa_p \right. \notag \\
                        & \hspace{1.1cm}\left. - \log \LP \sum_{m=0}^{\ddiamond}\frac{\kappa_q^m}{m!} \RP
                                               + \log \LP \sum_{m=0}^{\ddiamond}\frac{\kappa_p^m}{m!} \RP \RB d\x \notag \\
                        &= \int_{\x} q(\x) \LB \kappa_q \mub_q'\x - \kappa_p \mub_p'\x + \dbullet\log(\kappa_q) - \dbullet\log(\kappa_p) \right. \notag \\
                        & \hspace{1.1cm}\left. - \log \LP \sum_{m=0}^{\ddiamond}\frac{\kappa_q^m}{m!} \RP
                                               + \log \LP \sum_{m=0}^{\ddiamond}\frac{\kappa_p^m}{m!} \RP \RB d\x \notag \\
&\mbox{Further simplifications:}\\
                        &\leq \int_{\x} q(\x) \LB \kappa_q \mub_q'\x - \kappa_p \mub_p'\x + \dbullet\log(\kappa_q) - \dbullet\log(\kappa_p) \right. \notag \\
                        & \hspace{1.1cm}\left. - \log \LP \sum_{m=0}^{\ddiamond}\frac{\kappa_q^m}{m!} \RP
                                               + \LP \sum_{m=0}^{\ddiamond} \log \frac{\kappa_p^m}{m!} \RP \RB d\x \notag \\
&\mbox{(by Jensen's inequality)}\\
                        &= \int_{\x} q(\x) \LB \kappa_q \mub_q'\x - \kappa_p \mub_p'\x + \dbullet\log(\kappa_q) - \dbullet\log(\kappa_p) \right. \notag \\
                        & \hspace{1.1cm}\left. + \log \LP \sum_{m=0}^{\ddiamond}\frac{\kappa_q^m}{m!} \RP 
                                               - \sum_{m=0}^{\ddiamond} \LP m \log (\kappa_p) - \log m! \RP \RB d\x \notag \\
                        &\leq \int_{\x} q(\x) \LB \kappa_q \mub_q'\x - \kappa_p \mub_p'\x + \dbullet\log(\kappa_q) - \dbullet\log(\kappa_p) \right. \notag \\
                        & \hspace{1.1cm}\left. + \log \LP \sum_{m=0}^{\ddiamond}\frac{\kappa_q^m}{m!} \RP 
                                               - \sum_{m=1}^{\ddiamond} \LP m \log (\kappa_p) - m \log m + m - 1 \RP \RB d\x \notag \\
&\mbox{(using $n \log \frac{n}{e} + 1 \leq \log n! \leq (n+1)\log\frac{n+1}{e} + 1$)}\\
                        &= \int_{\x} q(\x) \LB \kappa_q \mub_q'\x - \kappa_p \mub_p'\x + \dbullet\log(\kappa_q) - \dbullet\log(\kappa_p) + \log \LP \sum_{m=0}^{\ddiamond}\frac{\kappa_q^m}{m!} \RP 
                                               \right. \notag \\
                        & \hspace{1.1cm}\left. - \sum_{m=1}^{\ddiamond} \LP m \log (\kappa_p) - m \log m \RP - \ddiamond(\ddiamond + 1) + (\ddiamond + 1) \RB d\x \notag \\
                        &= \int_{\x} q(\x) \LB \kappa_q \mub_q'\x - \kappa_p \mub_p'\x + \dbullet\log(\kappa_q) - \dbullet\log(\kappa_p) + \log \LP \sum_{m=0}^{\ddiamond}\frac{\kappa_q^m}{m!} \RP 
                                               \right. \notag \\
                        & \hspace{1.1cm}\left. - \sum_{m=1}^{\ddiamond} \LP m \log (\kappa_p) - m \log m \RP - \ddiamond^2 + 1 \RB d\x \notag \\
                        &= \int_{\x} q(\x) \LB \kappa_q \mub_q'\x - \kappa_p \mub_p'\x + \dbullet\log(\kappa_q) - \dbullet\log(\kappa_p) + \log \LP \sum_{m=0}^{\ddiamond}\frac{\kappa_q^m}{m!} \RP 
                                               \right. \notag \\
                        & \hspace{1.1cm}\left. - \sum_{m=1}^{\ddiamond} \LP m \log (\kappa_p) \RP + \ddiamond(\ddiamond + 1)\log \ddiamond - \ddiamond^2 + 1 \RB d\x \notag \\
                        &= \int_{\x} q(\x) \LB \kappa_q \mub_q'\x - \kappa_p \mub_p'\x + \dbullet\log(\kappa_q) - \dbullet\log(\kappa_p) + \log \LP \sum_{m=0}^{\ddiamond}\frac{\kappa_q^m}{m!} \RP 
                                               \right. \notag \\
                        & \hspace{1.1cm}\left. - \ddiamond(\ddiamond + 1)\log (\kappa_p) + \ddiamond(\ddiamond + 1)\log \ddiamond - \ddiamond^2 + 1 \RB d\x \notag \\
                        &= \int_{\x} q(\x) \LB \kappa_q \mub_q'\x - \kappa_p \mub_p'\x + \dbullet\log(\kappa_q) - \dbullet\log(\kappa_p) + \log \LP \sum_{m=0}^{\ddiamond}\frac{\kappa_q^m}{m!} \RP 
                                               \right. \notag \\
                        & \hspace{1.1cm}\left. -\LP \frac{{\left( d-3\right) }^{2}}{4} + \frac{d-3}{2}\RP \log (\kappa_p) + \ddiamond(\ddiamond + 1)\log \ddiamond - \ddiamond^2 + 1 \RB d\x \notag \\
                        &= \int_{\x} q(\x) \LB \kappa_q \mub_q'\x - \kappa_p \mub_p'\x + \dbullet\log(\kappa_q) + \log \LP \sum_{m=0}^{\ddiamond}\frac{\kappa_q^m}{m!} \RP 
                                               \right. \notag \\
                        & \hspace{1.1cm}\left. -\LP \frac{d^2 - 2d + 1}{4} \RP \log (\kappa_p )+ \ddiamond(\ddiamond + 1)\log \ddiamond - \ddiamond^2 + 1 \RB d\x \notag \\
                        &= \int_{\x} q(\x) \LB \kappa_q \mub_q'\x - \kappa_p \mub_p'\x + \dbullet\log(\kappa_q) + \log \LP 1 + \sum_{m=1}^{\ddiamond}\frac{\kappa_q^m}{m!} \RP 
                                               \right. \notag \\
                        & \hspace{1.1cm}\left. -\LP \frac{d^2 - 2d + 1}{4} \RP \log (\kappa_p )+ \ddiamond(\ddiamond + 1)\log \ddiamond - \ddiamond^2 + 1 \RB d\x \notag \\
                        &\leq \int_{\x} q(\x) \LB \kappa_q \mub_q'\x - \kappa_p \mub_p'\x + \dbullet\log(\kappa_q) + \sum_{m=1}^{\ddiamond}\frac{\kappa_q^m}{m!} 
                                               \right. \notag \\
                        & \hspace{1.1cm}\left. -\LP \frac{d^2 - 2d + 1}{4} \RP \log (\kappa_p )+ \ddiamond(\ddiamond + 1)\log \ddiamond - \ddiamond^2 + 1 \RB d\x \notag \\
&\mbox{(using $n \geq \log (1 + n) \geq \frac{n}{1+n}, ~(n > -1)$)}\\
                        &= \int_{\x} q(\x) \LB \kappa_q \mub_q'\x - \kappa_p \mub_p'\x + \dbullet\log(\kappa_q) + \sum_{m=1}^{\ddiamond}\frac{\kappa_q^m}{m!} 
                                               \right. \notag \\
                        & \hspace{1.1cm}\left. -\LP \frac{d^2 - 2d + 1}{4} \RP \log (\kappa_p )+ \ddiamond(\ddiamond + 1)\log \ddiamond - \ddiamond^2 + 1 \RB d\x \notag \\
                        &= \kappa_q - \kappa_p \mub_p'\mub_q + \dbullet\log(\kappa_q) + \sum_{m=1}^{\ddiamond}\frac{\kappa_q^m}{m!}  \notag \\
                        & \hspace{1.1cm} -\LP \frac{d^2 - 2d + 1}{4} \RP \log (\kappa_p )+ \ddiamond(\ddiamond + 1)\log \ddiamond - \ddiamond^2 + 1 \notag \\
&\mbox{(as $\int_{\x} q(\x) = 1$, and $\EE[\x] = \mub_q$, and $\mub_q'\mub_q = 1$)}
\end{align}
}
\end{proof}

The term $\mub_q'\mub_p$ can be seen as the cosine distance between the prior and postieror mean vectors. 
For $0 < \kappa_q < 1$, the term $\sum_{m=1}^{\ddiamond}\frac{\kappa_q^m}{m!} \geq \kappa_q$. However for large $\kappa_q$ and large $d$ this term can grow very large.

\subsection*{Special case: uniform prior}
Since the \ac{VMF} distribution is defined on the $S^{d-1}$, hypersphere, which is actually a specific case of a Stiefel manifold where $r=1$ is the radius. The Stiefel manifold has finite area,
\begin{align}
 \tau(d,r) = \frac{2^r\pi^{\frac{pr}{2}}}{\pi^{\frac{r(r-1)}{4}}\prod_{j=1}^{r}{\Gamma\LP \frac{p-j+1}{2} \RP}},
\end{align}
and so,
\begin{align}
 \tau(d,1) = \frac{2\pi^{\frac{p}{2}}}{\Gamma\LP \frac{p}{2} \RP},
\end{align}

For the special case of the uniform prior (more precisely $\lim_{\kappa_p \to 0}$), the prior \ac{PDF} reduces to,
\begin{align}
    M_d(\mub,\kappa) &= c_d(0)e^{0} \notag\\
                     &= \frac{\Gamma \LP \frac{d}{2} \RP}{2\pi^{\frac{d}{2}}},
\end{align}
which is simply one over the area on the manifold. This leads to a simpler form for the \ac{KL}-divergence.
\begin{corollary}
 For prior and posterior distributions as defined above over vectors $\x \in \Reals^d, \LN \x \RN_2 = 1, d < \infty$, we have
\begin{align}
 \KL(q(\x) || p(\x)) &= \kappa_q - \dstar \log 2 \notag \\
\end{align}
\end{corollary}
\begin{proof}
{\allowdisplaybreaks 
\begin{align}
    &\KL(q(\x) || p(\x)) = \int_{\x} q(\x) \log \frac{q(\x)}{p(\x)} d\x, \notag \\
                        &= \int_{\x} q(\x) \LB \log c_d(\kappa_q)e^{\kappa_q\mub_q'\x} - \log c_d(0) \RB d\x, \notag \\
                        &= \int_{\x} q(\x) \LB \kappa_q\mub_q'\x + \log c_d(\kappa_q) - \log \Gamma \LP \frac{d}{2} \RP + \log \LP 2\pi^{\frac{d}{2}} \RP \RB d\x, \notag \\
                        &= \int_{\x} q(\x) \LB \kappa_q\mub_q'\x + \log c_d(\kappa_q) - \log(\dstar)! + (d/2) \log \LP 2\pi \RP \RB d\x, \notag \\                        
                        &= \int_{\x} q(\x) \LB \kappa_q\mub_q'\x + \dstar\log(\kappa_q) - (d/2)\log(2\pi)\right. \notag\\
                        &\hspace{1.1cm}\left. - \log I_{\dstar}(\kappa_q) - \log(\dstar)! + (d/2) \log \LP 2\pi \RP \RB d\x, \notag \\
                        &= \int_{\x} q(\x) \LB \kappa_q\mub_q'\x + \dstar\log(\kappa_q) - \log I_{\dstar}(\kappa_q) - \log(\dstar)! \RB d\x, \notag \\
                        &= \int_{\x} q(\x) \LB \kappa_q\mub_q'\x + \dstar\log(\kappa_q) - \log\LP \frac{\kappa_q}{2}\RP^{\dstar} + \log \Gamma\LP\frac{d}{2}\RP - \log(\dstar)! \RB d\x, \notag \\
                        &= \int_{\x} q(\x) \LB \kappa_q\mub_q'\x + \dstar\log(\kappa_q) - \dstar \log\LP \frac{\kappa_q}{2}\RP \RB d\x, \notag \\
                        &= \int_{\x} q(\x) \LB \kappa_q\mub_q'\x - \dstar \log 2 \RB d\x, \notag \\
                        &= \kappa_q - \dstar \log 2, \notag \\
\end{align}
}
\end{proof}

For this special case, it can be seen that the dependence on the dimension is much more benign. This could prove useful for further computation (\eg if the \ac{KL}-divergence were to be used in a \ac{PAC}-Bayes bound \cite{Langford05}).



\section{Conclusions}
We have presented a derivation of the \acf{KL}-divergence for the \acf{VMF}-distribution, including the special case of a uniform prior over the hypersphere.

\bibliographystyle{plainnat}
\bibliography{kl_vmf}

\end{document}